\documentclass[runningheads]{llncs}

\usepackage[utf8]{inputenc}
\usepackage{dsfont}
\usepackage{amsmath}
\usepackage{amssymb}
\usepackage{rotating}
\usepackage{multirow}
\usepackage{array}
\usepackage{subfigure}
\usepackage{epsfig}
\usepackage{calc}
\usepackage{latexsym}
\usepackage{algorithm}
\usepackage{algcompatible}
\usepackage{graphics}
\usepackage{color}
\usepackage{tikz}

\def\Re{\mathds{R}}

\def\Za{\mathds{Z}}

\def\Bo{\mathds{B}}

\begin{document}

\title{Efficient Hill-Climber for Multi-Objective Pseudo-Boolean Optimization}
\titlerunning{Efficient Hill-Climber for Multi-Objective Pseudo-Boolean Optimization}

\author{Francisco Chicano\inst{1} \and Darrell Whitley\inst{2} \and Renato Tinós \inst{3}}

\authorrunning{F. Chicano, D. Whitley and R. Tinós} 
%

\institute{Dept. de Lenguajes y Ciencias de la Computaci\'on, University of M\'alaga, Spain
\email{chicano@lcc.uma.es} \and
Dept. of Computer Science, Colorado State University,
Fort Collins CO, USA
\email{whitley@cs.colostate.edu} \and
Department of Computing and Mathematics, University of S\~ao Paulo, Brazil
\email{rtinos@ffclrp.usp.br}
}

\maketitle              

\begin{abstract}
Local search algorithms and iterated local search algorithms
are a basic technique. Local search can be a stand along search methods,
but it can also be hybridized with evolutionary algorithms.
Recently, it has been shown that it is possible to identify
improving moves in Hamming neighborhoods for
$k$-bounded pseudo-Boolean optimization problems in constant time.
This means that local search does not need to enumerate neighborhoods
to find improving moves. It also means that evolutionary algorithms
do not need to use random mutation as a operator, except perhaps
as a way to escape local optima. In this paper, we show
how improving moves can be identified in constant time for multiobjective problems 
that are expressed as $k$-bounded pseudo-Boolean functions. 
In particular, multiobjective forms of NK Landscapes and Mk Landscapes
are considered.  

\end{abstract}

\section{Introduction}
\label{sec:introduction}

Local search and iterated local search algorithms~\cite{Hoos2004SLS} start at an initial solution 
and then search for an improving move based on a notion of a neighborhood
of solutions that are adjacent to the current solution.    This paper will consider
$k$-bounded pseudo-Boolean functions,  where the Hamming distance 1 neighborhood
is the most commonly used local search neighborhood.    

Recently,  it has been shown that the location of improving moves can be calculated
in constant time for the Hamming distance 1 ``bit flip" neighborhood~\cite{DBLP:conf/gecco/WhitleyC12}.
This has implications for both local search algorithms as well as simple evolutionary algorithms such as the
(1+1) Evolution Strategy.      Since we can calculate the location of improving moves,  we
do not need to enumerate neighborhoods to discover improving moves.   

Chicano et al.  \cite{Chicano2014gecco} generalize this result to present a local search algorithm that 
explore the solutions contained in a Hamming ball of radius $r$ around a solution in constant time.
This means that evolutionary
algorithms need not use mutation to find improving moves;   either mutation should be used to
make larger moves (that flip more than $r$ bits),  or mutation should be used to enable a form of restarts.  
It can also makes crossover more important.    Goldman et al.~\cite{Goldman2015} combined 
local search that automatically calculates the location of improving moves in constant time with recombination 
to achieve globally optimal results on relatively large Adjacent NK Landscape problems (e.g. 10,000 variables).

Whitley \cite{Whitley2015gecco} has introduced the notion of Mk Landspaces to replace NK Landscapes.
Mk Landscapes are  $k$-bounded pseudo-Boolean optimization problems composed of a linear combination of
$M$ subfunctions, where each subfunction is a pseudo-Boolean optimization problem defined
over $k$ variables.    This definition is general enough to include NK landscapes,  MAX-kSAT, as well
as spin glass problems.

In this paper,  we extend these related concepts to multi-objective optimization.    We define
a class of multi-objective Mk Landscapes and show how these generalize over previous definitions
of multi-objective NK Landscapes.     We also show how exact methods can be used to select
improving moves in constant time.   In the multi-objective space, the notion of an ``improving 
move" is complex because improvement can be improvement in all objectives,  or improvement
in only part of the objectives.        When there are improvement in all objectives,  then clearly
the improvement should be accepted.   However,  when there are improvement in only a subset
of objectives,   it is less clear what moves should be accepted because it is possible for search
algorithms to cycle and to visit previously discovered solutions.     Methods are proposed that
allow the identification of improving moves in constant time for multi-objective optimization.
Methods are also proposed to prevent local search algorithms from cycling and thus repeatedly
revisiting previously discovered solutions.  The results of this work could also be introduced in existing local search algorithms for multi-objective optimization, like Anytime Pareto Local Search~\cite{DuboisLacoste2015}.

The rest of the paper is organized as follows. In the next section we introduce 
Multi-objective pseudo-Boolean optimization problems.    
Section~\ref{sec:scores} defines the ``Scores" of a solution.    The Score vector tracks changes in the evaluation
function and makes it possible to track the locations of improving moves.   An algorithm is introduced to track multiple
Scores and to efficiently update them for multi-objective optimization.  Section~\ref{sec:rball} considers how to address the
problems of selecting improving moves in a multi-objective search space when the move only improves some,
but not all,  of the objectives. Section~\ref{sec:experiments} empirically evaluates the proposed algorithms.  
Section~\ref{sec:conclusions} summarizes the conclusions and outline the potential for future work.

\section{Multi-Objective Pseudo-Boolean Optimization}
\label{sec:pbo}

In this paper we consider pseudo-Boolean vector functions with $k$-bounded epistasis, where the component functions are \emph{embedded landscapes}~\cite{Heckendorn1999} or \emph{Mk Landscapes}~\cite{Whitley2015gecco}. We will extend the concept of Mk Landscapes to the multi-objective domain and, thus, we will base our nomenclature in that of Whitley~\cite{Whitley2015gecco}.

\begin{definition}[Vector Mk Landscape]
Given two constants $k$ and $d$, a \emph{vector Mk Landscape} $\mathbf{f} : \Bo^n \to \Re^d$ is a $d$-dimensional vector pseudo-Boolean function defined over $\Bo^n$ whose components are Mk Landscapes. That is, each component $f_i$ can be written as a sum of $m_i$ subfunctions, each one depending at most on $k$ input variables\footnote{In general, we will use \textbf{boldface} to denote vectors in $\Re^d$, as $\mathbf{f}$, but we will use normal weight for vectors in $\Bo^n$, like $x$.}:
\begin{equation}
\label{eqn:embedded}
f_i(x) = \sum_{l=1}^{m_i} f_i^{(l)}(x) ~~~~ \text{for $1 \leq i \leq d$,}
\end{equation}
where the subfunctions $f_i^{(l)}$ depend only on $k$ components of $x$. 
\end{definition}

This definition generalizes that of Aguirre and Tanaka~\cite{AguirreTanaka2004} for MNK Landscapes. In Figure~\ref{fig:vig-function} we show a vector Mk Landscape with $d=2$ dimensions. The first objective function, $f_1$, can be written as the sum of 5 subfunctions, $f_1^{(1)}$ to $f_1^{(5)}$. The second objective function, $f_2$, can be writte as the sum of 3 subfunctions, $f_2^{(1)}$ to $f_2^{(3)}$. All the subfunctions depend at most on $k=2$ variables.

It could seem that the previous class of functions is restrictive because each subfunction depends on a bounded number of variables. However, every compressible pseudo-Boolean function can be transformed in polynomial time into a quadratic pseudo-Boolean function (with $k=2$)~\cite{Rosenberg1975}.

A useful tool for the forthcoming analysis is the \emph{co-ocurrence graph}~\cite{Crama1990} $G = (V,E)$, where $V$ is the set of Boolean variables and $E$ contains all the pairs of variables $(x_{j_1},x_{j_2})$ that \emph{co-occur} in a subfunction $f_i^{(l)}$ for any $1 \leq i \leq d$ and $1 \leq l \leq m_i$ (both variables are arguments of the subfunction). In other terms, two variables $x_{j_1}$ and $x_{j_2}$ co-occur if there exists a subfunction mask $w_{i,l}$ where the $j_1$-th and $j_2$-th bits are 1.
In Figure~\ref{fig:nk-example} we show the subfunctions of a vector Mk Landscape with $2$-bounded epistasis and its corresponding variable co-occurrence graph.

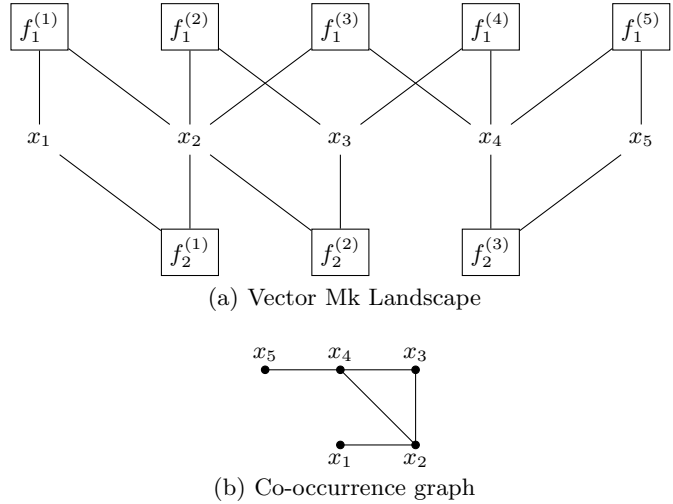
\begin{figure}[!tb]
\centering
\subfigure[Vector Mk Landscape]{
\label{fig:vig-function}
\begin{tikzpicture}[scale=1]
\node (x1) at (0,0) {$x_1$};
\node (x2) at (2,0) {$x_2$};
\node (x3) at (4,0) {$x_3$};
\node (x4) at (6,0) {$x_4$};
\node (x5) at (8,0) {$x_5$};

\node[rectangle,draw=black] (f11) at (0,1.5) {$f_1^{(1)}$};
\node[rectangle,draw=black] (f12) at (2,1.5) {$f_1^{(2)}$};
\node[rectangle,draw=black] (f13) at (4,1.5) {$f_1^{(3)}$};
\node[rectangle,draw=black] (f14) at (6,1.5) {$f_1^{(4)}$};
\node[rectangle,draw=black] (f15) at (8,1.5) {$f_1^{(5)}$};

\node[rectangle,draw=black] (f21) at (2,-1.5) {$f_2^{(1)}$};
\node[rectangle,draw=black] (f22) at (4,-1.5) {$f_2^{(2)}$};
\node[rectangle,draw=black] (f23) at (6,-1.5) {$f_2^{(3)}$};

\draw [-] (x1) -- (f11);
\draw [-] (x2) -- (f11);
\draw [-] (x2) -- (f12);
\draw [-] (x2) -- (f13);
\draw [-] (x3) -- (f12);
\draw [-] (x3) -- (f14);
\draw [-] (x4) -- (f13);
\draw [-] (x4) -- (f14);
\draw [-] (x4) -- (f15);
\draw [-] (x5) -- (f15);

\draw [-] (x1) -- (f21);
\draw [-] (x2) -- (f21);
\draw [-] (x2) -- (f22);
\draw [-] (x3) -- (f22);
\draw [-] (x5) -- (f23);
\draw [-] (x4) -- (f23);

\end{tikzpicture}
}
\subfigure[Co-occurrence graph]{
\label{fig:vig-vig}
\begin{minipage}{5cm}
\centering
\begin{tikzpicture}
\filldraw (0,0) circle (0.05);
\filldraw (0,1) circle (0.05);
\filldraw (1,0) circle (0.05);
\filldraw (1,1) circle (0.05);
\filldraw (-1,1) circle (0.05);

\draw (0,0) node[anchor=north] {$x_1$};
\draw (1,0) node[anchor=north] {$x_2$};
\draw (1,1) node[anchor=south] {$x_3$};
\draw (0,1) node[anchor=south] {$x_4$};
\draw (-1,1) node[anchor=south] {$x_5$};

\draw (0,0) -- (1,0);
\draw (1,0) -- (1,1);
\draw (1,1) -- (0,1);
\draw (0,1) -- (1,0);
\draw (0,1) -- (-1,1);
\end{tikzpicture}
\end{minipage}
}
\caption{A vector Mk Landscape with $k=2$, $n=5$ variables and $d=2$ dimensions (top) and its corresponding co-occurrence graph (bottom).}
\label{fig:nk-example}
\end{figure}

We will consider, without loss of generality, that all the objectives (components of the vector function) are to be maximized. Next, we include the definition of some standard multi-objective concepts to make the paper self-contained.

\begin{definition}[Dominance]
Given a vector function $\mathbf{f} : \Bo^n \to \Re^d$, we say that solution $x \in \Bo^n$ \emph{dominates} solution $y \in \Bo^n$, denoted with $x \succ_{\mathbf{f}} y$,  if and only if $f_i(x) \geq f_i(y)$ for all $1\leq i\leq d$ and there exists $j \in \{1, 2, \ldots, d\}$ such that $f_j(x) > f_j(y)$. When the vector function is clear from the context, we will use $\succ$ instead of $\succ_{\mathbf{f}}$.
\end{definition}

\begin{definition}[Pareto Optimal Set and Pareto Front]
Given a vector function $\mathbf{f} : \Bo^n \to \Re^d$, the \emph{Pareto Optimal Set} is the set of solutions $P$ that are not dominated by any other solution in $\Bo^n$. That is:
\begin{equation}
P = \left\{ x\in \Bo^n \middle| \nexists y\in \Bo^n, y \succ x \right\} .
\end{equation}
The \emph{Pareto Front} is the image by $\mathbf{f}$ of the Pareto Optimal Set: $PF = \mathbf{f} (P)$.
\end{definition}

\begin{definition}[Set of Non-dominated Solutions]
Given a vector function $\mathbf{f} : \Bo^n \to \Re^d$, we say that a set $X \subseteq \Bo^n$ is a \emph{set of non-dominated solutions} when there is no pair of solutions $x, y \in X$ where $y \succ x$, that is, $\forall x \in X, \nexists y\in X, y \succ x$.
\end{definition}

\begin{definition}[Local Optimum~\cite{Paquete2007}]
\label{def:lo}
Given a vector function $\mathbf{f} : \Bo^n \to \Re^d$, and a neighborhood function $N: \Bo^n \rightarrow 2^{\Bo^n}$, we say that solution $x$ is a \emph{local optimum} if it is not dominated by any other solution in its neighborhood: $\nexists y \in N(x), y \succ x$.
\end{definition}

\section{Moves in a Hamming Ball}
\label{sec:scores}

We can characterize a \emph{move} in $\Bo^n$ by a binary string $v \in \Bo^n$ having 1 in all the bits that change in the solution. Following~\cite{Chicano2014gecco} we will extend the concept of \emph{Score}\footnote{What we call \emph{Score} here is also named $\Delta$-evaluation by other authors~\cite{Taillard1991}.} to vector functions.

\begin{definition}[Score]
For $v, x \in \Bo^n$, and a vector function $\mathbf{f} : \Bo^n \to \Re^d$, we denote the \emph{Score} of $x$ with respect to \emph{move} $v$ as $\mathbf{S}_v(x)$, defined as follows:
\begin{equation}
\label{eqn:score-def}
\mathbf{S}_v(x) = \mathbf{f} (x \oplus v) - \mathbf{f}(x),
\end{equation}
where $\oplus$ denotes the exclusive OR bitwise operation (sum in $\Za_2$).
\end{definition}

The Score $\mathbf{S}_v(x)$ is the change in the vector function when we move from solution $x$ to solution $x \oplus v$, that is obtained by flipping in $x$ all the bits that are 1 in $v$. Our goal is to efficiently decide where to move from the current solution. If possible, we want to apply \emph{improving} moves to our current solution. While the concept of ``improving'' move is clear in the single-objective case (an improving move is one that increases the value of the objective function), in multi-objective optimization any of the $d$ component functions could be improving, disimproving or neutral. Thus, we need to be more clear in this context, and define what we mean by ``improving'' move.   It is useful to define two kinds of improving moves: the \emph{weak} improving moves and the \emph{strong} improving moves. The reason for this distinction will be clear in Section~\ref{sec:rball}.

\begin{definition}[Strong and Weak Improving Moves]
\label{def:weakstrong}
Given a solution $x \in \Bo^n$, a move $v \in \Bo^n$ and a vector function $\mathbf{f} : \Bo^n \to \Re^d$, 
we say that move $v$ is a \emph{weak improving move} if there exists $i\in \{1,2, \ldots,d\}$ such that $f_i(x \oplus v) > f_i(x)$. 
We say that move $v$ is a \emph{strong improving move} if it is a weak improving move and for all $j \in \{1, 2, \ldots,d\}$
$f_j(x \oplus v) \geq f_j(x)$. 
\end{definition}

Using our definition of Score, we can say that a move $v$ is a weak improving move if there exists a $j\in \{1, 2, \ldots,d\}$ for which $S_{j,v}(x) > 0$. It is a strong improving  move if $S_{i,v}(x) \geq 0$ for all $i \in \{1, 2, \ldots,d\}$ and there exists a $j\in \{1, 2, \ldots,d\}$ for which $S_{j,v}(x) > 0$.

From Definition~\ref{def:weakstrong} it can be noticed that if $v$ is a strong improving move in $x$ then $x \oplus v \succ x$, that is, the concept of strong improving move coincides with that of dominance. 
It can also be noticed that in the single-objective case, $d=1$, both concepts are the same.
Strong improving moves are clearly desirable, since they cannot be disimproving for any objective and they will improve at least one. Weak improving moves, on the other hand, improve at least one objective but could disimprove other ones. 

In particular, if $v$ is a weak, but not strong, improving move in solution $x$, then it will improve at least one objective, say $i$-th, and disimprove at least another one, say $j$-th. If this move is taken, in the new solution, $x \oplus v$, the same move $v$ will be again a weak, but not strong, improving move. However, now $v$ will improve (at least) the $j$-th objective and will disimprove (at least) $i$-th. Taking $v$ again in $x \oplus v$ will lead to $x$, and the algorithm cycles. Thus, any hill climber taking weak improving moves should include a mechanism to avoid cycling.

Scores are introduced in order to efficiently identify where the (weak or strong) improving moves are. For this purpose, we can have a data structure where all the improving moves can be accessed in constant time. 
As the search progresses the Score values change and they also move in the data structure to keep improving moves separated from the rest.
A naïve approach to track all improving moves in a Hamming Ball of radius $r$ around a solution would require to store all possible Scores for moves $v$ with $|v| \leq r$, where $|v|$ denotes the number of 1 bits in $v$.

If we naively use equation (\ref{eqn:score-def}) to explicitly update the scores, 
we will have to evaluate all $\sum_{i=1}^{r}\binom{n}{i} = O(n^r)$ neighbors in the Hamming ball.
Instead, if the objective function is a vector Mk Landscape fulfilling some requirements described in Theorem~\ref{thm:memory}, we can design an efficient next improvement hill climber for the radius $r$ neighborhood that only stores a linear number of Scores and requires a constant time to update them.

\subsection{Scores Update}
\label{subsec:scores-update}

Using the fact that each component $f_i$ of the objective vector function is an Mk Landscape, we can write:
\begin{equation}
\label{eqn:score-epistasis}
S_{i,v}(x) = \sum_{l=1}^{m_i} \left( f_i^{(l)}(x \oplus v) - f_i^{(l)}(x) \right) = \sum_{l=1}^{m_i} S_{i,v}^{(l)}(x),
\end{equation}
where we use $S_{i,v}^{(l)}$ to represent the score of the subfunction $f_i^{(l)}$ for move $v$. Let us define $w_{i,l} \in \Bo^n$ as the binary string such that the $j$-th element of $w_{i,l}$ is 1 if and only if $f_i^{(l)}$ depends on variable $x_j$.
The vector $w_{i,l}$ can be considered as a mask that characterizes the variables that affect $f_i^{(l)}$. Since $f_i^{(l)}$ has bounded epistasis $k$, the number of ones in $w_{i,l}$, denoted with $|w_{i,l}|$, is at most $k$. By the definition of $w_{i,l}$, the next equalities immediately follow.
\begin{align}
\label{eqn:f-masked}
f_i^{(l)}(x \oplus v) &= f_i^{(l)}(x) ~~~ \text{for all $v \in \Bo^n$ with $v \wedge w_{i,l} =0$}, \\
\label{eqn:scoring-reduction}
S_{i,v}^{(l)}(x) &= \left\{ 
\begin{array}{ll}
0 & \text{if $w_{i,l} \wedge v = 0$,} \\ 
S_{i, v\wedge w_{i,l}}^{(l)}(x) & \text{otherwise.}
\end{array} 
\right.
\end{align}

Equation (\ref{eqn:scoring-reduction}) claims that if none of the variables that change in the move characterized by $v$ is an argument of $f_i^{(l)}$ the Score of this subfunction is zero, since the value of this subfunction will not change from $f_i^{(l)}(x)$ to $f_i^{(l)}(x\oplus v)$. On the other hand, if $f_i^{(l)}$ depends on variables that change, we only need to consider for the evaluation of $S_{i,v}^{(l)}(x)$ the changed variables that affect $f_i^{(l)}$. These variables are characterized by the mask vector $v \wedge w_{i,l}$. With the help of (\ref{eqn:scoring-reduction}) we can re-write (\ref{eqn:score-epistasis}):
\begin{equation}
\label{eqn:scores2}
S_{i,v}(x) = \sum_{l=1 \atop w_{i,l} \wedge v \neq 0}^{m_i} S_{i, v \wedge w_{i,l}}^{(l)}(x).
\end{equation}

Equation~\eqref{eqn:scores2} simply says that we don't have to consider all the subfunctions to compute a Score. This can reduce the run time to compute the scores from scratch. 

During the search, instead of computing the Scores using~(\ref{eqn:scores2}) after every move, it is more efficient in time to store the Scores $\mathbf{S}_v(x)$ of the current solution $x$ in memory and update only those that are affected by the move. 

In the following, and abusing of notation, given a move $v \in \Bo^n$ we will also use $v$ to represent the set of variables that will be flipped in the move (in addition to the binary string).

%

For each of the Scores to update, the change related to subfunction $f_i^{(l)}$ can be computed with the help of $S_{i,v}^{(l)}(x \oplus t) = f_i^{(l)} (x \oplus t \oplus v) - f_i^{(l)}(x \oplus t)$ and $S_{i,v}^{(l)}(x) = f_i^{(l)} (x \oplus v) - f_i^{(l)}(x)$. The component $S_{i,v}$ will be updated by subtracting $S_{i,v}^{(l)}(x)$ and adding $S_{i,v}^{(l)}(x \oplus t)$. This procedure is shown in Algorithm~\ref{alg:update}, where the term $S_{i,v}$ represents the $i$-th component of the Score of move $v$ stored in memory and $M^r$ is the set of moves whose scores are stored. In the worst (and naïve) case $M^r$ is the set of all strings $v$ with at most $r$ ones, $M^r=\{v | 1 \leq |v| \leq r\}$, and $|M^r| = O(n^r)$. However, we will prove in Section~\ref{subsec:scores-decomposition} that, for some vector Mk Landscapes, we only need to store $O(n)$ Scores to identify improving moves in a ball of radius $r$.

\begin{algorithm}[!ht]
\begin{algorithmic}[1]
\REQUIRE $S, x, t$
\FOR{$(i,l)$ such that $w_{i,l} \wedge t \neq 0$}
\label{lin:for-l}
\FOR{$v \in M^r$ such that $w_{i,l} \wedge v \neq 0$}
\label{lin:for-v}
\STATE $S_{i,v} \leftarrow S_{i,v} + f_i^{(l)} (x \oplus t \oplus v) - f_i^{(l)}(x \oplus t)$
\STATEx \hspace{52pt} $- f_i^{(l)} (x \oplus v) + f_i^{(l)}(x)$
\label{lin:s-update}
\ENDFOR
\ENDFOR
\end{algorithmic}
\caption{Efficient algorithm for Scores update}
\label{alg:update}
\end{algorithm}

\subsection{Scores Decomposition}
\label{subsec:scores-decomposition}

Some scores can be written as a sum of other scores. The benefit of such a decomposition is that we do not really need to store all the scores in memory to have complete information of the influence that the moves in a Hamming ball of radius $r$ have on the objective function $\mathbf{f}$. The co-occurrence graph has a main role in identifying the moves whose Scores are fundamental to recover all the improving moves in the Hamming ball.

Let us denote with $G[v]$ the subgraph of $G$ induced by $v$, that is, the subgraph containing only the vertices in $v$ and the edges of $E$ between vertices in $v$.

\begin{proposition}[Score decomposition]
\label{prop:decomposition}
Let $v_1, v_2 \in \Bo^n$ be two moves such that $v_1 \cap v_2 = \emptyset$ and variables in $v_1$ do not co-occur with variables in $v_2$. In terms of the co-occurrence graph this implies that there is no edge between a variable in $v_1$ and a variable in $v_2$ and, thus, $G[v_1 \cup v_2] =  G[v_1] \cup G[v_2]$. Then the score function $\mathbf{S}_{v_1 \cup v_2}(x)$ can be written as:
\begin{equation}
\mathbf{S}_{v_1 \cup v_2}(x) = \mathbf{S}_{v_1}(x) + \mathbf{S}_{v_2}(x)
\end{equation}
\end{proposition}
\begin{proof}
Using \eqref{eqn:scores2} we can write:
\begin{align*}
S_{i,v_1 \cup v_2}(x) &= \sum_{l=1 \atop w_{i,l} \wedge (v_1 \vee v_2) \neq 0}^{m_i} S_{i, (v_1 \vee v_2) \wedge w_{i,l}}^{(l)}(x) \\
&= \sum_{l=1 \atop (w_{i,l} \wedge v_1) \vee (w_{i,l} \wedge v_2) \neq 0}^{m_i} S_{i, (v_1 \wedge w_{i,l}) \vee (v_2 \wedge w_{i,l})}^{(l)}(x).
\end{align*}

Since variables in $v_1$ do not co-occur with variables in $v_2$, there is no $w_{i,l}$ such that $v_1 \wedge w_{i,l} \neq 0$ and $v_2 \wedge w_{i,l} \neq 0$ at the same time. Then we can write:
\begin{align*}
S_{i,v_1 \cup v_2}(x)
&= \sum_{l=1 \atop w_{i,l} \wedge v_1 \neq 0}^{m_i} S_{i, v_1 \wedge w_{i,l}}^{(l)}(x) +
\sum_{l=1 \atop w_{i,l} \wedge v_2 \neq 0}^{m_i} S_{i, v_2 \wedge w_{i,l}}^{(l)}(x)
= S_{i,v_1}(x) + S_{i,v_2}(x),
\end{align*}
and the result follows.
\qed
\end{proof}

For example, in the vector Mk Landscape of Figure~\ref{fig:nk-example} the scoring function $\mathbf{S}_{\underline{1,3,4}}$ can be written as the sum of the scoring functions $\mathbf{S}_{\underline{1}}$ and $\mathbf{S}_{\underline{3,4}}$, where we used $\underline{i_1,i_2, ...}$ to denote the binary string having 1 in positions $i_1, i_2, \ldots$, and the rest set to 0.

A consequence of Proposition~\ref{prop:decomposition} is that we only need to store scores for moves $v$ where $G[v]$ is a connected subgraph. If $G[v]$ is not a connected subgraph, then there are sets of variables $v_1$ and $v_2$ such that $v=v_1 \cup v_2$ and $v_1 \cap v_2 = \emptyset$ and, applying Proposition~\ref{prop:decomposition} we have $\mathbf{S}_v(x)=\mathbf{S}_{v_1}(x) + \mathbf{S}_{v_2}(x)$. Thus, we can recover all the scores in the Hamming ball of radius $r$ from the ones for moves $v$ where $1 \leq |v| \leq r$ and $G[v]$ is connected. In the following we will assume that the set $M^r$ of Algorithm~\ref{alg:update} is:
\begin{equation}
\label{eqn:mr}
M^r = \left\{v \in \Bo^n \middle| 1 \leq |v| \leq r \text{ and } G[v] \text{ is connected}  \right\}.
\end{equation}

\subsection{Memory and Time Complexity of Scores Update}

We will now address the question of how many of these Scores exist and what is the cost in time of updating them after a move.

%
%
%
%

\begin{lemma}
\label{lem:connected-subgraphs}
Let $\mathbf{f}:\Bo^n \rightarrow \Re^d$ be a vector Mk Landscape where each Boolean variable appears in at most $c$ subfunctions $f_i^{(l)}$. Then, the number of connected subgraphs with size no greater than $r$ of the co-occurrence graph $G$  containing a given variable $x_j$ is $O((3 c k)^r)$.
\end{lemma}
\begin{proof}
For each connected subgraph of $G$ containing $x_j$ we can find a spanning tree with $x_j$ at the root. 
The degree of any node in $G$ is bounded by $c k$, since each variable appears at most in $c$ subfunctions and each subfunction depends at most on $k$ variables.
Given a tree of $l$ nodes with $x_j$ at the root, we have to assign variables to the rest of the nodes in such a way that two connected nodes have variables that are adjacent in $G$. The ways in which we can do this is bounded by $(c k)^{l-1}$. We have to repeat the same operation for all the possible rooted trees of size no greater than $r$. If  $T_l$ is the number of rooted trees with $l$ vertices, then the number of connected subgraphs of $G$ containing $x_j$ and with size no greater than $r$ nodes is bounded by
\begin{equation}
\sum_{l=1}^{r} T_l (c k)^{l-1} \leq \sum_{l=1}^{r} 3^l (c k)^{l-1} \leq 3 (3 c k)^r,
\end{equation}
where we used the result in~\cite{Otter1948} for the asymptotic behaviour of $T_l$:
\begin{equation}
\lim_{l \rightarrow \infty} \frac{T_l}{T_{l-1}} \approx 2.955765.
\end{equation}
\qed
\end{proof}

%
%
%
%
%
%
%
%
%
%

Lemma~\ref{lem:connected-subgraphs} provides a bound for the number of moves in $M^r$ that contains an arbitrary variable $x_j$. In effect, the connected subgraphs in $G$ containing $x_j$ corresponds to the moves in $M^r$ that flip variable $x_j$. An important consequence is given by the following theorem.

\begin{theorem}
\label{thm:memory}
Let $\mathbf{f}:\Bo^n \rightarrow \Re^d$ be a vector Mk Landscape where each Boolean variable appears in at most $c$ subfunctions. Then, the number of connected subgraphs of $G$ of size no greater than $r$ 
is $O(n (3ck)^r)$, which is linear in $n$ if $c$ is independent of $n$. This is the cardinality of $M^r$ given in \eqref{eqn:mr}.
\end{theorem}
\begin{proof}
The set of connected subgraphs of $G$ with size no greater than $r$ is the union of connected subgraphs of $G$ of size no greater than $r$ that contains each of the $n$ variables. According to Lemma~\ref{lem:connected-subgraphs} the cardinality of this set must be $O(n (3ck)^r)$.
\qed
\end{proof}

The next Theorem bounds the time required to update the scores. 

\begin{theorem}
\label{thm:runtime}
Let $\mathbf{f}:\Bo^n \rightarrow \Re^d$ be a vector Mk Landscape where each Boolean variable appears in at most $c$ subfunctions $f_i^{(l)}$. The time required to update the Scores using Algorithm~\ref{alg:update} is $O(b(k) |t| (3ck)^{r+1})$ where $b(k)$ is a bound on the time required to evaluate any subfunction $f_i^{(l)}$.
\end{theorem}
\begin{proof}
Since each variable appears in at most $c$ subfunctions, the number of subfunctions containing at least one of the bits in $t$ is at most $c |t|$, and this is the number of times that the body of the outer loop starting in Line~\ref{lin:for-l} of Algorithm~\ref{alg:update} is executed. 
Once the outer loop has fixed a pair $(i,l)$, the number of moves $v \in M^r$ with $w_{i,l} \wedge v \neq 0$ is the number of moves $v \in M^r$ that contains a variable in $w_{i,l}$. Since $|w_{i,l}| \leq k$ and using Lemma~\ref{lem:connected-subgraphs}, this number of moves is $O(k(3ck)^r)$. Line~\ref{lin:s-update} of the algorithm is, thus, executed $O(|t| c k (3ck)^r)$ times, and considering the bound on the time to evaluate the subfunctions, $b(k)$ the result follows.
\qed
\end{proof}

Since $|t| \leq r$, the time required to update the Scores is $\Theta(1)$ if $c$ does not depend on $n$. Observe that if $c$ is $O(1)$, then the number of subfunctions of the vector Mk Landscape is $m = \sum_{i=1}^{d} m_i = O(n)$. On the on the hand, if every variable appears in at least one subfunction (otherwise the variable could be removed), $m = \Omega(n)$. Thus, a consequence of $c = O(1)$ is that $m = \Theta(n)$.

\section{Multi-Objective Hamming-Ball Hill Climber}
\label{sec:rball}

We have seen that, under the hypothesis of Theorem~\ref{thm:memory}, a linear number of Scores can provide information of all the Scores in a Hamming ball of radius $r$ around a solution. However, we need to sum some of the scores to get complete information of where all the improving moves are, and this is not more efficient than exploring the Hamming ball. In order to efficiently identify improving moves we have to discard some of them. In particular, we will discard all the improving moves whose scores are not stored in memory. In~\cite{Chicano2014gecco} the authors proved for the single-objective case that if none of the $O(n)$ stored scores is improving, then it cannot exist an improving move in the Hamming ball of radius $r$ around the current solution. Although not all the improving moves can be identified, it is possible to identify local optima in constant time when the hill climber reaches them. This is a desirable property for any hill climber. We will prove in the following that this result can be adapted to the multi-objective case.

If one of the scores stored indicates a strong improving move, then it is clear that the hill climber is not in a local optima, and it can take the move to improve the current solution. However, if only weak improving moves can be found in the Scores store, it is not possible to certify that the hill climber reached a local optima. The reason is that two weak improving moves taken together could give a strong improving move in the Hamming ball. For example, let us say that we are exploring a Hamming ball of radius $r=2$, variables $x_1$ and $x_2$ do not co-occur in a two-dimensional vector function, and $\mathbf{S}_{\underline{1}}=(-1, 3)$ and $\mathbf{S}_{\underline{2}}=(3, -1)$. Moves $\underline{1}$ and $\underline{2}$ are weak improving moves, but the move $\mathbf{S}_{\underline{1,2}} = \mathbf{S}_{\underline{1}} + \mathbf{S}_{\underline{2}} = (2,2)$ is a strong improving move. We should not miss that strong improving move during our exploration.

To discover all strong improving moves in the Hamming ball we have to consider weak improving moves. But we saw in Section~\ref{sec:scores} that taking weak improving moves is dangerous because they could make the algorithm to cycle. One very simple and effective mechanism to avoid cycling is to classify weak  improving moves according to a weighted sum of their score components. 

\begin{definition}[$\mathbf{w}$-improving move and $\mathbf{w}$-score]
Let $\mathbf{f}:\Bo^n \rightarrow \Re^d$ be a vector Mk Landscape, and $\mathbf{w} \in \Re^d$ a $d$-dimensional weight vector. We say that a move $v \in \Bo^n$ is $\mathbf{w}$-improving for solution $x$ if $\mathbf{w} \cdot \mathbf{S}_v(x) > 0$, where $\cdot$ denotes the dot product of vectors. We call $\mathbf{w} \cdot \mathbf{S}_v(x)$ the $\mathbf{w}$-score of move $v$ for solution $x$.
\end{definition}

\begin{proposition}
\label{prop:atleast}
Let $\mathbf{f}:\Bo^n \rightarrow \Re^d$ be a vector Mk Landscape, and $\mathbf{w} \in \Re^d$ a $d$-dimensional weight vector with $w_i > 0$ for $1 \leq i \leq d$. If there exists a strong improving move in a ball of radius $r$ around solution $x$, then there exists $v \in M^r$ such that $\mathbf{w} \cdot \mathbf{S}_v > 0$.
\end{proposition}
\begin{proof}
Let us say that $v$ is a strong improving move in the Hamming ball of radius $r$. Then there exist moves $v_1, v_2, \ldots v_j \in M^r $ such that $\mathbf{S}_v=\sum_{l=1}^j \mathbf{S}_{v_l}$. Since $v$ is strong improving and all $w_i > 0$, we have $\mathbf{w} \cdot \mathbf{S}_v = \sum_{l=1}^j \mathbf{w} \cdot \mathbf{S}_{v_l} > 0$. There must be a $v_l$ with $1 \leq l \leq j$ such that $\mathbf{w} \cdot \mathbf{S}_{v_l} > 0$.
\qed
\end{proof}

Proposition~\ref{prop:atleast} ensures that we will not miss any strong improving move in the Hamming ball if we take the weak improving moves with an improving $\mathbf{w}$-score. Thus, our proposed Hill Climber, shown in Algorithm~\ref{alg:rball}, will select strong improving moves in first place (Line~\ref{lin:select-move}) and $\mathbf{w}$-improving moves when no strong improving moves are available (Line~\ref{lin:select-wmove}). In this last case, we should report the value of solution $x$, since it could be a non-dominated solution (Line~\ref{lin:report}).
The algorithm will stop when no $\mathbf{w}$-improving move is available. In this case, a local optima has been reached, and we should report this final (locally optimal) solution (Line~\ref{lin:final-report}). The algorithm cannot cycle, since only $\mathbf{w}$-improving moves are selected, and this means that an improvement is required in the direction of $\mathbf{w}$. A cycle would require to take a $\mathbf{w}$-disimproving move at some step of the climb.

\begin{algorithm}[!ht]
\begin{algorithmic}[1]
\REQUIRE scores vector $\mathbf{S}$, weight vector $\mathbf{w}$, initial solution $x$
\ENSURE local optimum in $x$ (and potentially non-dominated intermediate solutions)
\STATE $\mathbf{S} \leftarrow$ computeScores($x$); 
\label{lin:scores-scratch}
\WHILE{$\mathbf{w} \cdot \mathbf{S}_v > 0$ for some $v\in M^r$}
\label{lin:innerloop}
	\IF{there is a strong improving move $v \in M^r$}
		\STATE $t \leftarrow $ selectStrongImprovingMove($\mathbf{S}$);
		\label{lin:select-move}
	\ELSE
		\STATE $t \leftarrow $ selectWImprovingMove($\mathbf{S}$);
		\label{lin:select-wmove}
		\STATE report($x$);
		\label{lin:report}		
	\ENDIF
	\STATE updateScores($\mathbf{S}$,$x$,$t$);
	\label{lin:update}
	\STATE $x \leftarrow x \oplus t$;
	\label{lin:move}
\ENDWHILE
\STATE report($x$);
\label{lin:final-report}
\end{algorithmic}
\caption{Multi-objective Hamming-Ball Hill Climber.}
\label{alg:rball}
\end{algorithm}

The procedure \texttt{report} in Algorithm~\ref{alg:rball} should add the reported solution to an external set of non-dominated solutions. This set should be managed by the high-level algorithm invoking the Hamming Ball Hill Climber.

For an efficient implementation of Algorithm~\ref{alg:rball}, the scores stored in memory can be classified in three categories, each one stored in a different bucket: strong  improving moves, $\mathbf{w}$-improving moves that are not strong improving moves, and the rest. The scores can be moved from one of the buckets to the other as they are updated. The move from one bucket to another requires constant time, and thus, the expected time per move  in Algorithm~\ref{alg:rball} is $\Theta(1)$, excluding the time required by \texttt{report}. This implementation corresponds to a next improvement hill climber. An approximate form of best improvement hill climber could also be implemented following the guidelines in~\cite{WhitHoweHains2013}.

The weight vector $\mathbf{w}$ in the hill climber determines a direction to explore in the objective space.
The use of $\mathbf{w}$ to select the weak improving moves is equivalent to consider improving moves of the single-objective function $\mathbf{w} \cdot \mathbf{f}$. However, there are two main reasons why it is more convenient to update and deal with the vector scores $\mathbf{S}$ rather than using scalar scores $S$ of $\mathbf{w} \cdot \mathbf{f}$. First, using vector scores we can identify strong improving moves stored in memory, while using scalar scores of $\mathbf{w} \cdot \mathbf{f}$ it is not possible to distinguish between weak and strong  improving moves. And second, it is possible to change $\mathbf{w}$ during the search without re-computing all the scores. The only operation to do after a change of $\mathbf{w}$ is a re-classification of the moves that are not strong improving\footnote{Distinguishing the weak, but not strong, improving moves from the strong disimproving moves in the implementation would reduce the runtime here, since only weak improving moves need to be re-classified.}.

Regarding the selection of improving moves in \texttt{selectStrongImprovingMove} and \texttt{selectWImprovingMove}, our implementation selects always a random one with the lowest Hamming distance to the current solution, that is, the move $t$ with the lowest value of $|t|$. As stated by Theorem~\ref{thm:runtime}, such moves are faster, in principle, than other more distant moves, since the time required for updating the Scores is proportional to $|t|$.

\section{Experimental Results}
\label{sec:experiments}

We implemented a simple Multi-Start Hill Climber algorithm to measure the runtime speedup of the proposed Multi-Objective Hamming Ball Hill Climber of Algorithm~\ref{alg:rball}. 
The algorithm iterates a loop where a solution and a weight vector are randomly generated and Algorithm~\ref{alg:rball} is executed starting on them. The algorithm keeps a set of non-dominated solutions, that is potentially updated whenever Algorithm~\ref{alg:rball} reports a solution. The loop stops when a given time limit is reached. In our experiments shown here this time limit was 1 minute. The machine used in all the experiments has an Intel Core 2 Quad CPU (Q9400) at 2,7 GHz,  3GB of memory and Ubuntu 14.04 LTS. Only one core of the Processor is used. The algorithm was implemented in Java 1.6.

To test the algorithm we have focused on MNK-Landscapes~\cite{AguirreTanaka2004}. An MNK-Landscape is a vector Mk Landscape where all $m_i=N$ for all $1 \leq i \leq d$ and each subfunction $f_i^{(l)}$ depends on $x_i$ and other $K$ more variables (thus, $k=K+1$). The subfunctions $f_i^{(l)}$ are randomly generated using real values between 0 and 1. In order to avoid inaccuracy problems with floating point arithmetic, instead of real numbers we use integer number between 0 and $q-1$ and the sum of subfunctions are not divided by $N$. That is, each component $f_i$ is an NKq-Landscape~\cite{Chen:2013:SOP:2463372.2463437}. We also focused on the \emph{adjacent model} of NKq-Landscape. In this model the variables each $f_i^{(l)}$ depends on are consecutive, that is, $x_i, x_{i+1}, \ldots, x_{i+K}$. This ensures that the number of subfunctions a given variable appears in is bounded by a constant, in particular, $K+1$, and Theorems~\ref{thm:runtime} and~\ref{thm:memory} apply.

\subsection{Runtime}
\label{subsec:runtime}

There are two procedures in the hill climber that requires $\Omega(n)$ time. The first one is a \emph{problem-dependent} initialization procedure, where the scores to be stored in memory are determined. This procedure is run only once in one run of the multi-start algorithm. In our experiments this time varies from 284 to 5,377 milliseconds. 

The second procedure is a \emph{solution-dependent} initialization of the hill climber starting from random solution and weight vector. This procedure is run once in each iteration of the multi-start hill climber loop, and can have an important impact on algorithm runtime, especially when there are no many moves during the execution of Algorithm~\ref{alg:rball}.
On the other hand, as the search progresses and the non-dominated set of solutions grows, the procedure to update it could also require a non-negligible run time that depends on the number of solutions in the non-dominated set, which could be proportional to the number of moves done during the search.

In Figure~\ref{fig:runtime} we show the average time per move in microseconds ($\mu$s) for the Multi-Start Hill Climber solving MNK-Landscapes where $N$ varies from $10,000$ to $100,000$, $q=100$, $K=3$, the dimensions are $d=2$ and $d=3$, and the exploration radius $r$ varies from 1 to 3. We performed 30 independent runs of the algorithm for each configuration, and the results are the average of these 30 runs. To compute the average, we excluded the time required by the problem-dependent initialization procedure.

\begin{figure}[!ht]
\centering
\includegraphics[width=12cm]{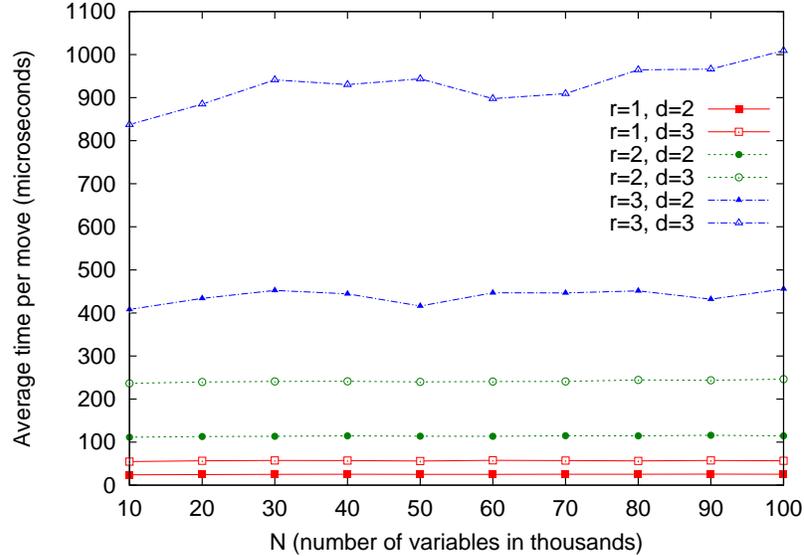}
\caption{Average time per move in $\mu$s for the Multi-Start Hill Climber based on Algorithm~\ref{alg:rball} for a MNK-Landscape with $d=2, 3$, $K=3$, $q=100$, $N=10,000$ to $100,000$ and $r=1$ to $3$.}
\label{fig:runtime}
\end{figure}

We can observe that moves are done very fast (tens to hundreds of microseconds). This is especially surprising if we consider the number of solutions ``explored'' in a neighborhood. For $N=100,000$ and $r=3$ the neighborhood contains around 166 trillion solutions that are explored in around 1 millisecond.
For all values of $r$ and $d$ the increase in the average time per move is very slow (if any) when $N$ grows.  This slight growth in the average run time is due to the solution-dependent initialization and the non-dominated set update, and contrasts with the theoretically $\Omega(n^r)$ time required by a black box algorithm. 

As we could expect, the value of $r$ has a great influence in the average time per move. In fact, the time is exponential in $r$. Regarding the memory required to store the Scores, we have already seen that it is $\Theta(n)$. In the particular case of the MNK-Landscapes with an adjacent interaction model and $r \leq N/K$ , it is not hard to conclude that the exact number of scores is $N(K^r-1)/(K-1)$, which is linear in $N$.

\subsection{Quality of the Solutions}
\label{subsec:quality}

In a second experiment we want to check if a large value of $r$ leads to better solutions. This highly depends on the algorithm that includes the hill climber. In our case, since the algorithm is a multi-start hill climber, we would expect an improvement in solution quality as we increase $r$. But at the same time, the average time per move is increased. Thus, there must be a value of $r$ at which the time is so large that lower values for the radius can lead to the same solution quality. In Figure~\ref{fig:attainments} we show the 50\%-empirical attainment surfaces  of the fronts obtained in the 30 independent runs of the multi-start hill climber for $N=10,000$, $d=2$, $q=100$ and $r$ varying from 1 to 3. The 50\%-empirical attainment surface (50\%-EAS) limits the region in the objective space that is dominated by half the runs of the algorithm. It generalizes the concept of \emph{median} to the multi-objective case (see~\cite{Knowles2005} for more details).

\begin{figure}[!ht]
\centering
\includegraphics[width=12cm]{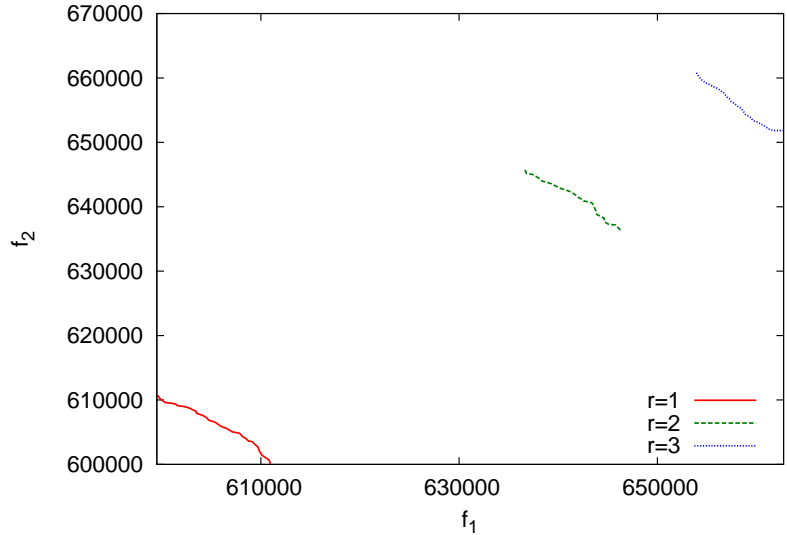}
\caption{50\%-empirical attainment surfaces of the 30 independent runs of the Multi-Start Hill Climber based on Algorithm~\ref{alg:rball} for a MNK-Landscape with $d=2$, $K=3$, $q=100$, $N=10,000$ and $r=1$ to $3$.}
\label{fig:attainments}
\end{figure}

We can see in Figure~\ref{fig:attainments} that the 50\%-EAS obtained for $r=2$ completely dominates the one obtained for $r=1$, and the 50\%-EAS for $r=3$ dominates that of $r=2$. That is, increasing $r$ we obtained better approximated Pareto fronts, even of the time per move is increased. This means that less moves are done int he given time limit (1 minute) but they are more effective.

\section{Conclusions and Future Work}
\label{sec:conclusions}

We proposed in this paper a hill climber based on an efficient mechanism to identify improving moves in a Hamming ball of radius $r$ around a solution of a $k$-bounded pseudo-Boolean multi-objective optimization problem. With this paper we contribute to an active line of research, sometimes called, Gray-Box optimization~\cite{Goldman2015}, that suggests the use of as much information of the problems as possible to provide better search methods, in contrast to the Black-Box optimization.

Our proposed hill climber performs each move in bounded constant time if the variables of the problem appears in at most a constant number of subfunctions. In practice, the experiments on adjacent MNK-Landscapes show that when $K=3$ and $d=2$, the average time per move varies from tenths to hundreds of microseconds when the exploration radius $r$ grows from 1 to 3. This number is independent of $n$ despite the fact that the hill climber is considering a Hamming Ball of radius $r$ with up to $\Theta(n^r)$ solutions.

Further work is needed to integrate this hill climber in a higher-level algorithm including mechanisms to escape from plateaus and local optima. On the other hand, one important limitation of our hill climber is that is does not take into account constraints in the search space. Constraint management and the combination with other components to build an efficient search algorithm seem two promising and challenging directions to work in the near future.

\bibliographystyle{splncs03}
\bibliography{landscapes}

\end{document}